\documentclass[lettersize,journal]{IEEEtran}
\usepackage{times}
\usepackage[numbers]{natbib}
\usepackage{multicol}
\usepackage{multirow}
\usepackage[bookmarks=true]{hyperref}	
\usepackage{amsmath}
\usepackage{mathtools}
\usepackage{paralist}
\usepackage{tabularx,algorithm}	
\usepackage{color}
\usepackage{fancyhdr}
\usepackage{algpseudocode}
\usepackage{bm,bbm}
\usepackage{tikz}
\usepackage{amsthm}
\usepackage{algpseudocode}
\usepackage{epstopdf}
\usepackage{amsfonts}
\usepackage{dsfont}
\usepackage{mathrsfs}
\usepackage{amssymb}
\usepackage{algorithm}
\usepackage{algcompatible}
\usepackage{epsfig}
\usepackage[english]{babel}
\usepackage[all]{xy}
\usepackage[latin1]{inputenc}
\usepackage{tikz}
\usepackage[T1]{fontenc}
\usepackage[labelformat=simple]{subcaption}

\def\thickhline{\noalign{\hrule height.8pt}}

\newtheorem{theorem}{Theorem}
\newtheorem{definition}{Definition}

\newtheorem{lemma}{Lemma}

\def\BibTeX{{\rm B\kern-.05em{\sc i\kern-.025em b}\kern-.08em
    T\kern-.1667em\lower.7ex\hbox{E}\kern-.125emX}}

\begin{document}
\title{Bias-Eliminated PnP for Stereo Visual Odometry: Provably Consistent and Large-Scale Localization}
\author{Guangyang Zeng*, Yuan Shen*, Ziyang Hong, Yuze Hong, Viorela Ila, Guodong Shi, Junfeng Wu
\thanks{Guangyang Zeng, Yuan Shen, Ziyang Hong, Yuze Hong, and Junfeng Wu are with the School of Data Science, Chinese
University of Hong Kong, Shenzhen 518172, China.

Viorela Ila and Guodong Shi are with the Australian Center for Robotics and
School of Aerospace, Mechanical and Mechatronic Engineering, University of Sydney, Australia.

* Equally contributed.}}

% \markboth{Journal of \LaTeX\ Class Files,~Vol.~18, No.~9, September~2020}%
% {How to Use the IEEEtran \LaTeX \ Templates}

\maketitle

\begin{abstract}
In this paper, we first present a bias-eliminated 
weighted (Bias-Eli-W)
perspective-n-point (PnP) estimator for stereo visual odometry (VO) with provable consistency. Specifically, leveraging statistical theory, we develop an asymptotically unbiased and $\sqrt {n}$-consistent PnP estimator that accounts for varying 3D triangulation uncertainties, ensuring that the relative pose estimate converges to the ground truth as the number of features increases.
Next, on the stereo VO pipeline side, we propose a framework that continuously triangulates contemporary features for tracking new frames, effectively decoupling temporal dependencies between pose and 3D point errors. We integrate the Bias-Eli-W PnP estimator into the proposed stereo VO pipeline, creating a synergistic effect that enhances the suppression of pose estimation errors.
We validate the performance of our method on the KITTI and Oxford RobotCar datasets.
Experimental results demonstrate that our method:
1) achieves significant improvements in both relative pose error and absolute trajectory error in large-scale environments;
2) provides reliable localization under erratic and unpredictable robot motions.
The successful implementation of the Bias-Eli-W PnP in stereo VO indicates the importance of information screening in robotic estimation tasks with high-uncertainty measurements, shedding light on diverse applications where PnP is a key ingredient.
\end{abstract}


\section{Introduction} \label{introduction}
\IEEEPARstart{V}{isual}  odometry (VO) refers to estimating the pose of a moving camera in a 3D space from sequential images captured by the camera. The significance of VO stems from its advantages of being infrastructure-free, cost-effective, lightweight, energy-efficient, etc~\cite{campos2021orb,merrill2023fast,jadidi2019continuous}. It enables robots to perceive and navigate their environment autonomously. 
Compared with monocular VO, stereo VO offers several advantages, such as scale consistency, better accuracy, and enhanced robustness, due to its ability to perceive depth directly~\cite{leutenegger2015keyframe,duan2022stereo}.

% Existing visual odometry (VO) methods typically optimize both camera poses and 3D map points simultaneously. \zgy{sean: explanation} This framework causes the errors of 3D points to be influenced not only by 2D feature matching noise but also by the errors in pose estimates, as shown in Figure~\ref{fig_motivation}. As a result, theoretical analysis and estimator optimization become complicated. Moreover, pose errors may grow more quickly due to this coupled and iterative process.
% Zhao \emph{et al.}~\cite{zhao2021feature} also proposed that when the covariance matrices of feature observation errors are isotropic, it would be better to solve the poses only. 

Existing VO methods typically optimize both camera poses and 3D map points simultaneously, with the map being used to track new frames through the perspective-n-point (PnP) algorithm~\cite{campos2021orb,ferrera2021ov,leutenegger2015keyframe}. These methods often lack accurate uncertainty estimation for point correspondences and neglect to incorporate estimator optimization and analysis to provide statistically grounded PnP estimates with theoretical guarantees. 
One of the key challenges in accurately estimating point uncertainties arises from the commonly adopted framework, where an initial frame is fixed, and all subsequent poses and 3D points are represented relative to this frame. In this setup, uncertainty estimation becomes inherently unreliable, as linear error propagation through the transformation chain becomes increasingly inaccurate~\cite{leutenegger2015keyframe}.
To address the issue of unbounded uncertainty growth, some methods have adopted a relative representation, where each frame is connected to its predecessor through a relative pose, and 3D points are anchored to the frame in which they are first observed~\cite{mei2011rslam,leutenegger2015keyframe}. While these approaches mitigated the problem of unbounded uncertainty growth, the simultaneous localization and mapping (SLAM) pipeline they employ introduces temporal coupling between pose and 3D point errors. This coupling hinders precise uncertainty estimation and can result in suboptimal relative pose estimation.

Interestingly, SOFT2~\cite{cvivsic2022soft2}, the top-ranked method on the KITTI odometry benchmark, adopts a relative pose representation and focuses solely on pure odometry rather than SLAM, i.e., it does not involve 3D point optimization. Although the authors did not explicitly mention this, it is a key factor contributing to its outstanding performance on the KITTI dataset, as odometry excels in minimizing relative pose error (RPE).
However, SOFT2 is a highly specialized solution tailored to autonomous driving scenarios, as it incorporates techniques specifically designed for ground vehicles. Furthermore, the algorithm has not yet been open-sourced, limiting its accessibility and reproducibility. 

Inspired by SOFT2, this paper focuses on a pure odometry framework that emphasizes relative pose estimation while excluding 3D point optimization. 
% This framework is expected to yield better RPE, a metric that is typically prioritized in odometry tasks and crucial for  reliable short-term decision-making in real-time applications~\cite{wang2022geometrically,liu2023automatic}. Furthermore, we delve deeper into the advantages of this framework, revealing that it breaks temporal error coupling, thereby facilitating uncertainty analysis and pose error suppression. Building on accurate relative pose estimation, we aim to incorporate sliding-window pose batch optimization to improve trajectory consistency and surpass the SLAM framework in terms of absolute trajectory error (ATE).
Specifically, we propose CurrentFeature Odometry, which leverages only the triangulated feature points from the current keyframe (KF) for PnP-based relative pose tracking. This framework effectively breaks the temporal coupling between pose and 3D point errors.
Building on this decoupling, we accurately model the uncertainties of the 3D points and optimize the estimator from a statistical perspective. This results in consistent relative pose estimation, i.e., the estimate converges to the true value as the number of points increases.
By integrating the consistent PnP estimator into the proposed stereo VO framework, we demonstrate that CurrentFeature Odometry not only achieves significantly lower RPE but also surpasses state-of-the-art (SOTA) SLAM algorithms in terms of absolute trajectory error (ATE).

% It is noteworthy that our algorithm is particularly well-suited for large-scale scenes, where global map points tend to become less accurate, and points from distant past frames are often difficult to identify and utilize effectively.

The main contributions of this paper are three-fold:

%\begin{figure}[!t]
 %   \centering
    %\includegraphics[width=0.9\linewidth]{example-image-a}
    %\caption{Two kinds of PnP tracking modes: one uses both historical and current features, and the other only utilizes current features. The blue points denote current features that are triangulated from the current keyframe, and the red points represent historical features stored in a global map. The uncertainties of current features only depend on 2D feature errors, while those of historical features are also affected by historical pose uncertainties.}
    %\label{fig_motivation}
%\end{figure}

\begin{itemize}
\item A bias-eliminated weighted (Bias-Eli-W) PnP estimator that accounts for varying 3D triangulation uncertainties is presented
 with provable consistency. Theoretical guarantees ensure that the relative pose estimate converges to the ground truth as the number of features increases.

    \item A novel stereo VO framework, CurrentFeature Odometry, is proposed. This framework utilizes only triangulated points from the current KF for PnP pose tracking, effectively breaking the temporal coupling between 3D triangulation and future pose estimation, which facilitates uncertainty propagation analysis and error suppression in pose estimation. 
    \item The proposed framework enables effective uncertainty analysis in feature matching and triangulation. The feature matching error covariance can be consistently estimated, offering flexibility to adapt to various feature detection and tracking schemes. Triangulation uncertainty is estimated in a manner tailored to feature point locations. These estimates, in turn, feed back into our Bias-Eli-W PnP.
\end{itemize}

Experimental results on the KITTI and Oxford RobotCar datasets  demonstrate that our method: 1) achieves significant
improvements ($24\%$ in RPE
 and $28\%$ in ATE on average compared to the second-best) on KITTI color sequences; 2) provides reliable localization
under erratic and unpredictable robot motions. 
Our code will be open-sourced at \url{https://github.com/LIAS-CUHKSZ/CurrentFeature-Odometry}.

% \textbf{Features}
% \begin{itemize}
%     \item \textbf{Linearly increased odometry error.} New frame tracking does not rely on a simultaneously maintained map, avoiding mutual error propagation between pose and map.
%         \item \textbf{Geometric-optimization-aided point weight learning.} Point weight learning is supervised by PnP pose loss. The factors that influence weighting include point distance, semantics, and texture.
%     \item \textbf{Consistent camera motion tracking.} 2D/3D point uncertainty estimation and bias-eliminated PnP estimator, resulting in consistent camera motion estimation. 
%     \item \textbf{Robust to both outliers and varied scenes.} RANSAC, weight learning, adaptive keyframe generation, and XXX. 
% \end{itemize}

\textbf{Notations:} For a vector $v$, $[v]_i$ denotes its $i$-th element, while $[v]_{i:j}$ represents the subvector containing its $i$-th to $j$-th elements. Throughout this paper, we follow the convention that the super-scripted $(\cdot)^o$ represents the true (or noise-free) value of a variable $(\cdot)$ 
and $\hat{(\cdot)}$ an estimate to $(\cdot)$.

\section{Related work} \label{related_work}

A large body of stereo VO algorithms is geometry based, leveraging geometric rigidity from classic computer vision models to estimate camera motion. Meanwhile, an increasing number of works use deep learning to infer motion directly from raw image sequences in an end-to-end manner. These include monocular VO in supervised learning pipelines ~\cite{wang2017deepvo,clark2017vinet,costante2020uncertainty}, unsupervised learning pipelines ~\cite{li2018undeepvo,zhan2018unsupervised}, and stereo VO ~\cite{teed2021droid}. Additionally, hybrid approaches integrate deep learning not only for feature detection and matching but also for feature tracking and depth estimation with geometry-based optimization in various contexts ~\cite{czarnowski2020deepfactors,yang2020d3vo,zhan2020visual,bescos2018dynaslam}.
In this paper, we primarily focus on geometry-based stereo VO, which can be broadly categorized into two types: feature-based methods and direct methods.

%\subsection{Feature-based methods} \label{feature_based_methods} 
Feature-based stereo VO methods typically detect and match feature points between consecutive frames, estimating motion by minimizing metrics such as reprojection errors~\cite{leutenegger2015keyframe,duan2022stereo,ferrera2021ov}. As reported in the literature~\cite{mur2017orb,warren2013high}, translation estimation is primarily influenced by 3D points close to the camera, while rotation estimation is sensitive to 3D points both near and far, with far points defined as having depths exceeding 40 times the stereo baseline. Consequently, for works tailored to ground vehicles~\cite{cvivsic2022soft2,de2016stereo}, detecting close ground features is critical for accurate translation estimation.
To address scenarios with sparse or poorly distributed points, some studies combine point and line features~\cite{gomez2019pl,gomez2016robust}. More recently, Fontan \emph{et al.}~\cite{fontan2024anyfeature} introduced an automated pipeline that can seamlessly switch between six different feature types, providing enhanced adaptability.

%\subsection{Direct methods} \label{derict_methods} 
Direct methods bypass the need for feature detection and instead optimize the pose directly by minimizing photometric errors~\cite{engel2015large,wang2017stereo,liu2017direct}. Leveraging pixel intensities in the image can enhance motion estimation accuracy. However, the reliance on the brightness constancy assumption may make these methods sensitive to noise. Additionally, estimating the depth for each pixel used often leads to computational inefficiency. To address these limitations, some works have combined the strengths of both feature-based and direct methods, resulting in semi-direct algorithms~\cite{forster2016svo,krombach2018feature}.

Another, less common, categorization method is based on the geometric relationship used for motion estimation~\cite{cvivsic2022soft2}. 3D-3D methods triangulate 3D points from the stereo current frame and use the iterative closest point algorithm to align them with the previous 3D point cloud~\cite{fraundorfer2011visual}. 2D-2D methods rely on epipolar constraints for pose tracking~\cite{cvivsic2022soft2}. However, the majority of works adopt a 3D-2D approach, where previous 3D points are aligned with current 2D points using the camera projection model and the pose to be estimated~\cite{pire2017s,yin2022dynam,qin2019general}.
Cvi\v{s}i\'c \emph{et al.}~\cite{cvivsic2022soft2} applied a pure 2D-2D method and achieved SOTA performance on the KITTI stereo odometry benchmark. However, 2D-2D algorithms are practically not robust to variations in translation distance or point distribution. The combination with a 3D-2D method will be a better choice.

% {\color{blue} \subsection{Uncertainty estimation} \label{uncertainty_est_methods}
% Finally, we want to briefly review the literature that involves uncertainty estimation in their framework~\cite{yin2022dynam,gomez2016robust,costante2020uncertainty,yang2020d3vo}. 
% }

\section{System overview} \label{system_overview}
The system overview of CurrentFeature Oodmetry is illustrated in Figure~\ref{system_overview}. In the front end, when a new stereo frame arrives, feature detection is first performed on a single image (in our case, the left image is chosen). Then, feature tracking with the KF and outlier rejection
are sequentially performed,
generating the point correspondences among the stereo KF and the left image of the current frame (CF). If the number of successfully tracked points falls below a predefined threshold or the distribution of the points degenerates, a new KF is created.
In the back end, 3D points are triangulated using the matched features in the stereo KF. The variance of 2D feature noise is also estimated, which is then used to calculate the uncertainties of the triangulated 3D points. With these elements, a bias-eliminated weighted (Bias-Eli-W) PnP estimator computes the pose of the CF relative to the KF. Finally, a sliding-window epipolar bundle adjustment (BA) is performed to refine the odometry.
In what follows, we will detail key modules involved.

\begin{figure}[!t]
    \centering
    \includegraphics[width=0.98\linewidth]{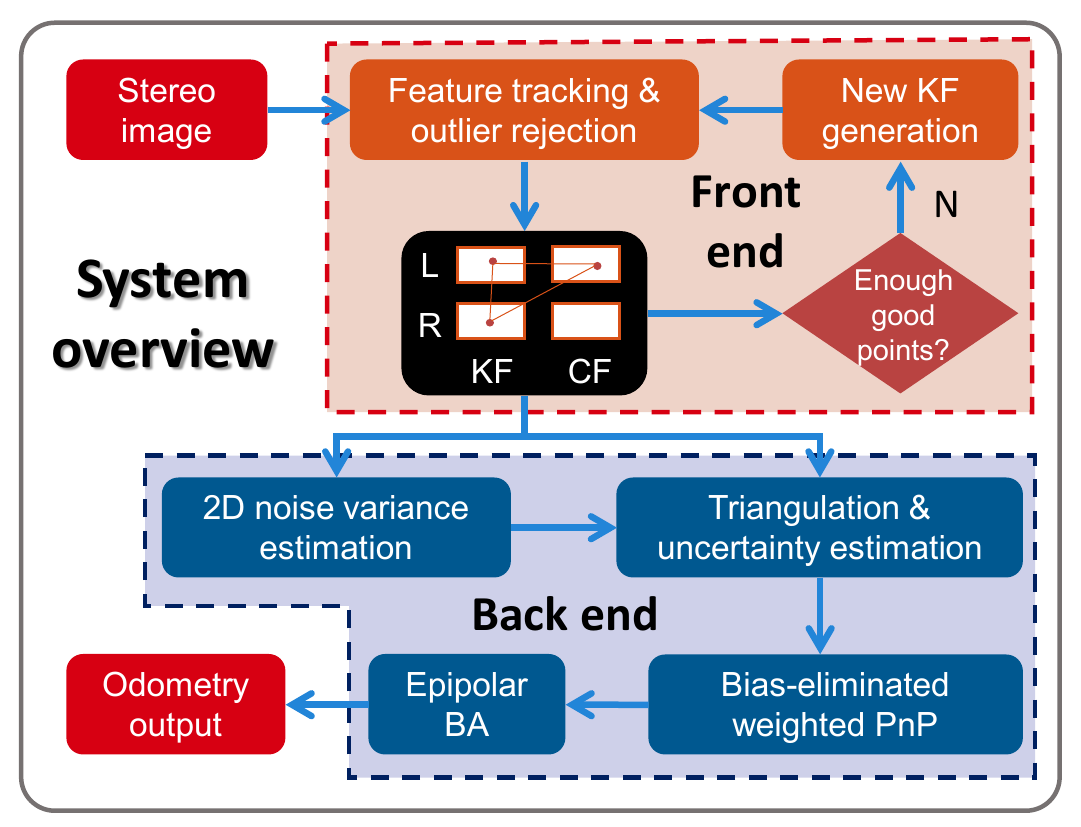}
    \caption{System overview. L and R refer to the left and right images, respectively, and KF and CF denote the keyframe and current frame, respectively.}
    \label{system_overview}
\end{figure}

\section{Front end} \label{front_end}

% The front end module serves as the initial processing pipeline, handling raw images and creating keyframes. \jf{Do we have mappling module? it may be confusing. subsequent mapping.} Similar to OV\textsuperscript{2}SLAM~\cite{ferrera2021ov}, the pipeline consists of several key components: image pre-processing, feature association, outlier removal, and keyframe generation. 

The front end module serves as the initial processing pipeline, handling raw images and creating KFs. Similar to OV\textsuperscript{2}SLAM~\cite{ferrera2021ov}, the pipeline consists of several key components: image pre-processing, feature association, outlier removal, and KF generation.

To enhance feature detection and tracking robustness, we first apply CLAHE~\cite{zuiderveld1994contrast} for image quality enhancement. Our tracking operates in a sliding-window manner, where each new frame is localized relative to the latest KF that serves as the local coordinate origin. Specifically, we track newly detected features using pyramidal Lucas-Kanade optical flow to establish 2D-2D matches between the left images of consecutive frames. As these 2D features are continuously tracked, their associated 3D coordinates in the KF are already stored. 
% These 3D points are utilized to estimate the CF's pose through a \zgy{RANSAC PnP} algorithm. In turn,  this initial pose estimate enhances 2D feature tracking by predicting the locations of remaining unmatched 3D points in newly arriving images. 

% To ensure robustness against potential mismatches, we implement a two-stage geometric verification strategy. In the first stage, we utilize the essential matrix within a RANSAC framework to validate the results of optical flow tracking. Leveraging the known camera extrinsics and the tracked 2D-2D correspondences between consecutive frames, this step efficiently identifies and removes outliers with minimal point requirements.
% \jf{should be Back end?Subsequently, for the PnP-based pose estimation, we employ our custom PnP-TLS robust estimator that iteratively refines the pose estimate while automatically identifying and down weighting outliers through a robust weighting scheme. This approach not only provides more reliable pose estimates} than RANSAC-based methods but also effectively handles matching uncertainties arising from various feature detection and tracking schemes.

To ensure robustness against potential mismatches, we implement a two-stage geometric verification strategy. We first apply the five-point algorithm with RANSAC to estimate the essential matrix and filter initial outliers from the tracked features. 
Second, using the triangulated 3D points in the KF, we perform an $\ell_1$-norm PnP to obtain a robust pose estimate and remove the 10\% of points with the largest reprojection errors.
New KFs are strategically inserted based on two primary criteria: tracking number and geometric quality. The system monitors both the number of successfully tracked features and the average feature displacement to determine KF insertion. When a new KF is created, additional features are detected using a uniform grid sampling strategy to maintain consistent feature distribution across the image, ensuring sufficient features for optical flow tracking in subsequent frames.

\section{Back end} \label{back_end}

\subsection{Triangulation and uncertainty estimation} \label{triangulation_and_uncertainty_est}
As illustrated in Figure~\ref{triangulation_pnp}, after feature tracking and outlier rejection, we obtain point correspondences among the KF and the left image of the CF, denoted as $\{x_i,y_i,z_i\}_{i=1}^{n}$, with $x_i$ in the right image of the KF, $y_i$ in the left image of the KF, and $z_i$ in the left image of the CF. Point coordinates are represented as normalized image coordinates. We suppose the error of feature matching follows a Gaussian distribution $\mathcal N(0,\sigma^2 I_2)$. In other words, we have 
\begin{equation*}
    x_i  = x_i^o+\epsilon_{x_i},~~y_i  = y_i^o+\epsilon_{y_i},
\end{equation*}
where $\epsilon_{x_i},\epsilon_{y_i} \sim \mathcal N(0,\sigma^2 I_2)$ are feature matching errors.
The isotropic error assumption is a good approximation for SLAM with range and bearing observations, including the pinhole camera model~\cite{zhao2021feature}. This assumption has been widely adopted in many existing studies~\cite{zhao2021feature,rosen2019se,wang2015dimensionality}.
According to Theorem 1 in~\cite{zeng2024consistent}, a consistent estimator $\hat \sigma^2$ for $\sigma^2$ can be constructed by solving an eigenvalue problem. This means that as the number of points increases, the 2D matching uncertainty estimate converges to the true value.

\begin{figure}[!htbp]
    \centering
    \includegraphics[width=0.86\linewidth]{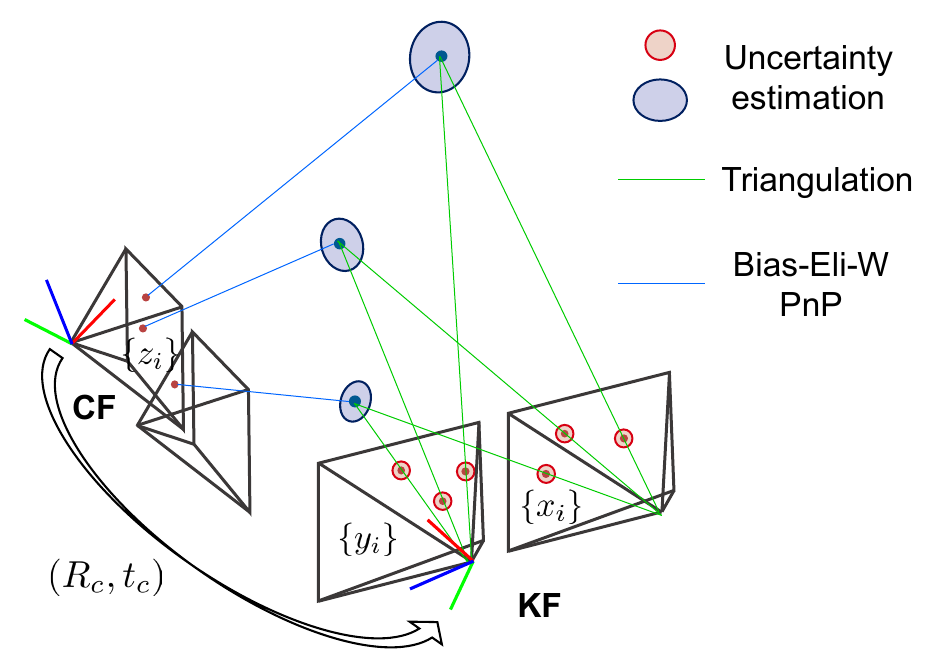}
    \caption{Illustration of frame tracking. The orange circles represent feature-matching uncertainties and the blue ellipses denote the triangulation uncertainties.}
    \label{triangulation_pnp}
\end{figure}

For triangulation, instead of the commonly used parallax-based triangulation~\cite{campos2021orb,ferrera2021ov} or singular-value-decomposition-based methods~\cite{hartley2003multiple}, we employ a linear least-squares closed-form solution, which is more general and better suited for uncertainty analysis. Specifically, suppose the coordinates of the 3D point in the left KF is $p_i$. We use $x_i^h$ and $y_i^h$ to denote the homogeneous coordinates. In the noise-free case, we have $y_i^{ho}$ parallel to $p_i$ and $x_i^{ho}$ parallel to $R_0^\top p_i-R_0^\top t_0$. Then, by involving errors, we obtain $(y_i^h-\epsilon_{y_i}) \times p_i=0$ and $(x_i^h-\epsilon_{x_i}) \times (R_0^\top p_i-R_0^\top t_0)=0$, where $(R_0,t_0)$ is the stereo baseline.
Letting \begin{equation*}
    A_i=\begin{bmatrix}
        y_i^{h \wedge} \\
        x_i^{h \wedge} R_0^\top
    \end{bmatrix},~~ b_i=\begin{bmatrix}
            0_3 \\
            x_i^{h \wedge} R_0^\top t_0
        \end{bmatrix},
\end{equation*}
where $(\cdot)^\wedge$ denotes the function that maps a vector to a skew-symmetric matrix, the above parallel geometric constraints have a compact form of $b_i=A_i p_i+\varepsilon_i$, in which $\varepsilon_i$ is the error term.
As a result, the least-squares triangulation formula is 
\begin{equation} \label{pt_triangulation}
    p_i=(A_i^\top A_i)^{-1} A_i^\top b_i.
\end{equation}

Since $A_i$ and $b_i$ contain noises $\epsilon_{x_i}$ and $\epsilon_{y_i}$, the exact uncertainty $\Sigma_i$ of $p_i$ is hard to obtain. We use the following uncertainty propagation formula to estimate:
\begin{equation} \label{3D_pt_uncertainty}
    \Sigma_i = J_{p_i} \Sigma J_{p_i}^\top,
\end{equation}
where $J_{p_i}$ is the Jacobian matrix of $p_i$ with respect to $\epsilon_i \triangleq [\epsilon_{x_i}^\top,\epsilon_{y_i}^\top]^\top$, and $\Sigma=\hat \sigma^2 I_4$ is the consistent estimate of the covariance of $\epsilon_i$. 
The error propagation in ~\eqref{3D_pt_uncertainty} is based on a first-order approximation. When the noise is sufficiently small, higher-order terms become negligible, allowing it to converge to the true covariance.
We observe that the adopted front end produces small feature-matching errors. For instance, the estimated noise standard deviation in the KITTI dataset is approximately $0.2$ pixels (see Section~\ref{experiments_real}). As a result, $\Sigma_i$ in~\eqref{3D_pt_uncertainty} can be considered highly accurate.
When the noise level is large, the first-order approximation no longer gives accurate results. In this case, the simulation-extrapolation, a simulation-based method for measurement error models developed in the statistics literature can be applied~\cite{cook1994simulation,stefanski1995simulation}. Further extension of the method for computational efficiency and examination of its validity are left for future research.

As illustrated in Figure~\ref{triangulation_pnp}, the uncertainty of a triangulated 3D point is heavily influenced by its depth. In general, the farther the point, the greater the uncertainty. When a point is too far, the stereo camera functionally degenerates into a monocular camera, losing the observability of the scale of the translation~\cite{campos2021orb}. By incorporating the estimated uncertainty, we can effectively reduce the impact of distant points.

\subsection{Bias-eliminated weighted PnP} \label{consistent_pnp}
On top of the triangulation, we use the 2D coordinates of the matched points in the CF to estimate the pose of the CF with respect to the KF, denoted as $(R_c,t_c)$. 
% Instead of using the constant-velocity model to generate the initial value for iterative PnP optimization as most of the existing methods do, we design L1 norm optimization and bias-eliminated least squares to produce better initial value.
Let $R_c=[r_1,r_2,r_3]^\top$, $t_c=[t_1,t_2,t_3]^\top$, and $\theta \triangleq \alpha[r_3^\top,r_1^\top,t_1,r_2^\top,t_2]^\top$, where $\alpha$ is a positive scale factor. By referring to (12) in~\cite{zeng2023cpnp}, we obtain a least-squares estimate for $\theta$ as follows:
\begin{equation} \label{biased_weighted_pnp}
    \hat \theta^{\rm B} = (H^\top H)^{-1} H^\top d, 
\end{equation}
% where 
% linear equation in~\cite{zeng2023cpnp}: 
% \begin{equation} \label{linear_pnp_meas}
%     d=H \theta + \varepsilon,
% \end{equation}
where $d=[z_1^\top,\ldots,z_n^\top]^\top$,
\begin{equation*}
    H=\begin{bmatrix}
        -z_1 \otimes (p_1-\bar p)^\top & I_2 \otimes [p_1^\top,1] \\
        \vdots & \vdots \\
        -z_n \otimes (p_n-\bar p)^\top & I_2 \otimes [p_n^\top,1]
    \end{bmatrix},
\end{equation*}
and $\bar p=\sum_{i=1}^{n}p_i/n$.  
% Then we solve the outlier-robust L1 norm optimization problem
% \begin{equation} \label{l1_optimization}
%     \min\limits_{\theta \in \mathbb R^{11}} \left\|H \theta-d \right\|_1.
% \end{equation}
% Problem~\eqref{l1_optimization} is a linear L1 optimization, whose global minimum can be found efficiently via linear programming~\cite{zibulevsky2010l1}. After solving~\eqref{l1_optimization}, we select points with residuals less than a threshold as inliers. With no loss of generality, we still denote the inliers as  

The work in~\cite{zeng2023cpnp} does not estimate and account for 3D point uncertainties. In our scenario, the triangulated point $p_i$ is affected by noise and associated with a covariance matrix $\Sigma_i$. As a result, the regressor matrix $H$ is correlated with the noise terms. According to estimation theory, this correlation generally renders the estimator $\hat{\theta}^{\rm B}$ neither asymptotically unbiased nor consistent~\cite{mu2017globally}.
The fundamental issue lies in the deviation of the mean of the term $p_i p_i^\top$, which is a component of $H^\top H$, from its noise-free counterpart $p_i^o p_i^{o\top}$. This deviation, quantified by $\Sigma_i$, introduces bias into the estimation process.
To eliminate the estimator bias, we need to subtract $\Sigma_i$ from $p_i p_i^\top$. This leads to the bias-eliminated solution 
\begin{equation} \label{unbiased_weighted_pnp}
    \hat \theta^{\rm BE} = \left(\frac{H^\top H}{n}-G \right)^{-1} \frac{H^\top d}{n}, 
\end{equation}
where 
$G=\frac{1}{n}\sum_{i=1}^{n} G_i^\top G_i$ with $G_i=[-z_i \otimes \Sigma_i^{\frac{1}{2}},I_2 \otimes [\Sigma_i^{\frac{1}{2}},0_{3 \times 1}] \,]$.
The recovery of $\hat R_c^{\rm BE}$ and $\hat t_c^{\rm BE}$ from $\hat \theta^{\rm BE}$ can refer to (14)-(17) in~\cite{zeng2023cpnp}.
Our bias-eliminated estimator $(\hat R_c^{\rm BE},\hat t_c^{\rm BE})$ owns the property of consistency. Before deriving the formal theorem, we introduce the definition of $\sqrt{n}$-consistency.
\begin{definition}[$\sqrt{n}$-Consistency in Probability]
     An estimator $\hat {\bm \gamma}$ is called a $\sqrt{n}$-consistent estimator of ${\bm \gamma}^o$ if $\hat {\bm \gamma}-{\bm \gamma}^o=O_p(1/\sqrt{n})$, i.e., for any $\varepsilon >0$, there exists a finite $M$ and a finite $N$ such that for any $n>N$, $\mathbb{P} (\|\sqrt{n}(\hat {\bm\gamma} -{\bm\gamma}^o)\|>M )<\varepsilon$.
\end{definition}
 The notion of ``$\sqrt{n}$-consistency'' includes two implications: The estimator is consistent, i.e., it converges to the true value as $n$ increases; The convergence is as fast as $1/\sqrt{n}$.

\begin{theorem} \label{consistency_be_est}
    The bias-eliminated estimator $(\hat R_c^{\rm BE},\hat t_c^{\rm BE})$ is $\sqrt{n}$-consistent, i.e., $\hat R_c^{\rm BE}-R_c^o=O_p(1/\sqrt{n}$), $\hat t_c^{\rm BE}-t_c^o=O_p(1/\sqrt{n})$.
\end{theorem}
\begin{proof}
     The proof is mainly based on the following lemma:
            \begin{lemma}{(\cite{zeng2022global}, Lemma 4):}\label{Lemma_1}
            Let $\{X_i\}$ be a sequence of independent random variables with $\mathbb{E}[X_i]=0$ and uniformly bounded $\mathbb{E}[X_i^2]$ for all $i$. Then, there holds $\sum_{i=1}^n X_i/n = O_p(1/\sqrt{n})$.
		\end{lemma}
    In the noise-free case, $(H^{o\top} H^o)^{-1} H^{o\top} d$, or equivalently $(\frac{H^{o\top} H^o}{n})^{-1} \frac{H^{o\top} d}{n}$ results in the ground truth $\theta^o$. The idea of the proof is to show that $\frac{H^{\top} H}{n}-G$ and $\frac{H^{\top} d}{n}$ converge to $\frac{H^{o\top} H^o}{n}$ and $\frac{H^{o\top} d}{n}$, respectively. Let $\Delta H=H-H^o$. 
    First, since $\Delta H$ only contains the first-order term of 3D point noise $\epsilon_{p_i}$, which has zero mean and finite covariance, based on Lemma~\ref{Lemma_1}, we have
    \begin{equation} \label{eqn:Op1}
        \frac{H^{\top} d}{n}-\frac{H^{o\top} d}{n}=\frac{\Delta H^{\top} d}{n}=O_p(\frac{1}{\sqrt{n}}).
    \end{equation}
    Second, $\Delta H^\top \Delta H$ contains both the first and second-order terms of $\epsilon_{p_i}$. For the second-order term $\epsilon_{p_i} \epsilon_{p_i}^\top$, we can subtract it by $\Sigma_i$ to achieve zero mean. It can be verified that $-G$ actually accomplishes this procedure. Hence, according to Lemma~\ref{Lemma_1}, it holds that 
    \begin{equation} \label{eqn:Op2}
        \begin{split}
             \frac{H^{\top} H}{n}-\frac{H^{o\top} H^{o}}{n} & =\frac{\Delta H^{\top} \Delta H}{n}+O_p(\frac{1}{\sqrt{n}}) \\
        & = G + O_p(\frac{1}{\sqrt{n}}).
        \end{split}
    \end{equation} 
    Finally, by combining~\eqref{eqn:Op1} and~\eqref{eqn:Op2}, we obtain
    \begin{align*}
        \hat \theta^{\rm BE} & = \left(\frac{H^{o\top} H^o}{n}+O_p(\frac{1}{\sqrt{n}})\right)^{-1} \left(\frac{H^{o\top} d}{n} +O_p(\frac{1}{\sqrt{n}})\right) \\
        & = \left(\frac{H^{o\top} H^o}{n}\right)^{-1} \frac{H^{o\top} d}{n} +O_p(\frac{1}{\sqrt{n}}) = \theta^o + O_p(\frac{1}{\sqrt{n}}).
    \end{align*}
    That is, $\hat \theta^{\rm BE}$ is a $\sqrt{n}$-consistent estimator of $\theta^o$. Since the recovery of $(\hat R_c^{\rm BE},\hat t_c^{\rm BE})$ from $\hat \theta^{\rm BE}$ involves only continuous functions, see equations (14)-(17) in~\cite{zeng2023cpnp}, $(\hat R_c^{\rm BE},\hat t_c^{\rm BE})$ is also $\sqrt{n}$-consistent estimator of $(R_c^{o},t_c^{o})$, which completes the proof.
\end{proof}

Note that in the bias-eliminated estimator~\eqref{unbiased_weighted_pnp}, the constraints on the rotation matrix are not applied. As a result, although the estimator is consistent, it may not own the minimum variance. Hence, we further take $(\hat R_c^{\rm BE},\hat t_c^{\rm BE})$ as the initial value and perform a weighted PnP iterative refinement. 
Let $h(\cdot)$ denote the pinhole camera projection model, that is, for a $p \in \mathbb R^3$, $h(p)=[p]_{1:2}/[p]_3$. The weight assigned for the $i$-th point is $\bar \Sigma_i^{-\frac{1}{2}}$, where $\bar \Sigma_i=J_{h_i} \Sigma_i J_{h_i}^\top$, and $J_{h_i}$ is the Jacobian matrix of $h(\hat R_c^{\rm BE} p_i+\hat t_c^{\rm BE})$ with respect to $p_i$.
As a result, the weighted residual for the $i$-th point is $Re_i=\bar \Sigma_i^{-\frac{1}{2}}\left(h(R_c p_i +t_c)-z_i \right)$, and we solve the following problem:
\begin{equation} \label{pnp_gn}
    (\hat R_c,\hat t_c) = \mathop{\rm arg~min}\limits_{(R_c,t_c)} \sum_{i=1}^{n} \rho_{\delta}\left(Re_{i}\right),
\end{equation}
where $\rho_{\delta}$ is the truncated least-squares (TLS) robust kernel function. Problem~\eqref{pnp_gn} is solved using an LM algorithm. 
Thanks to the consistency of the initial value $(\hat R_c^{\rm BE},\hat t_c^{\rm BE})$ and the quadratic convergence speed of the LM algorithm near the global minimum, when the point number $n$ is large, a single LM iteration is sufficient to achieve the minimum variance, which is computationally efficient~\cite{zeng2024consistent}.

\subsection{Epipolar bundle adjustment} \label{epipolar_BA}
When a new KF is generated, we perform a local BA, which includes the latest two KFs and intermediate ordinary frames (OFs). As shown in Figure~\ref{fig_epipolar_ba}, the parameters to be refined are six-degree relative poses $\xi_k \in \mathbb R^6, k=1,\ldots,K+1$. Each $\xi_k$ consists of three Euler angles and a translation vector. Since feature tracking in OFs only involves left images, their right images are not utilized in the local BA. 
It is worth noting that most approaches simultaneously optimize camera poses and the coordinates of visible 3D points in local BA by imposing reprojection constraints. However, in the context of odometry, the 3D points are latent variables that are not of primary interest. Including them in the optimization process may not necessarily improve pose estimation accuracy, but it does introduce additional computational overhead.
As an alternative, the epipolar constraint does not involve 3D points and directly connects two matched feature points through the relative pose. 
Suppose the relative pose is $(R,t)$, then, for two matched feature points $x$ and $y$, the epipolar  constraint is
\begin{equation} \label{epipolar_constraint}
    y^{h \top}E x^h=0,
\end{equation}
where $E = t^\wedge R$ is the essential matrix. 
Let $l = E x^h$, referred to as the epipolar line. Following SOFT2~\cite{cvivsic2022soft2}, we use the point-to-epipolar-line distance as the residual:
\begin{equation} \label{epipolar_residual}
    Re=\frac{y^{h \top} l}{\|[l]_{1:2}\|}.
\end{equation}

In~\cite{cvivsic2022soft2}, only the left images from the frames within the sliding window are involved in the BA optimization. This approach, which relies solely on temporal rigidity without baseline constraint, cannot refine the scale. In this paper, by additionally incorporating rigidity induced from the baseline (using the right images of keyframes), simultaneous refinement of all six degrees of freedom can be achieved.

Let $\xi=[\xi_1^\top,\ldots,\xi_{K+1}^\top]^\top$. We search for epipolar constraints for any pair of images involved and solve the following problem:
\begin{equation} \label{epipolar_ba}
    \hat \xi = \mathop{\rm arg~min}\limits_{\xi} \sum_{j \in \mathcal E} \sum_{i=1}^{n_j} \rho_{\delta}\left(Re_{i,j}\right),
\end{equation}
where $\mathcal E$ is the set of all image pairs, and $n_j$ is the number of feature matches in the $j$-th pair. Problem~\eqref{epipolar_ba} is solved using an LM algorithm. 

\begin{figure}[!htbp]
    \centering
    \includegraphics[width=0.8\linewidth]{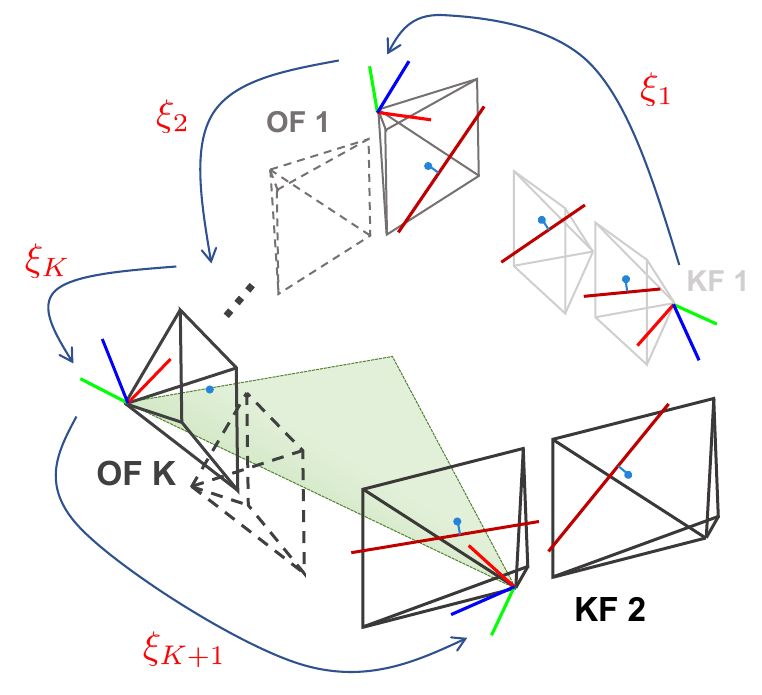}
    \caption{Illustration of epipolar BA. For two KFs, we utilize stereo images, while for OFs, we only use the left image. Epipolar constraints for each pair of images with matched features are involved in BA, and point-to-epipolar-line distances are optimized. }
    \label{fig_epipolar_ba}
\end{figure}

\section{Simulations and Experiments} \label{experiments}

\subsection{Simulation} \label{simulation}
We perform simulations for three purposes: 1) verifying our bias-eliminated PnP estimator is consistent (Theorem~\ref{consistency_be_est}); 2) proving that only using 3D points triangulated from the current KF leads to better odometry accuracy; 3) showing the effect of the epipolar BA. 
Throughout the simulations, we set the baseline of the stereo camera as $R=I$ and $t=[0.5,0,0]^\top$ m. We suppose the camera can see 3D points with depth within $[1,40]$ m. The focal length is set as $f=800$ pixels, and the principle point offset is $u_0=320$ pixels, $v_0=240$ pixels.
Unless otherwise specified, the standard deviation of feature-matching noises is 1 pixel.

\subsubsection{Consistency of bias-eliminated PnP} Recall that PnP tracking estimates the relative pose between CF and the latest KF. We let the number of feature correspondences vary from 30, 60, 120, 240, 480, and 960, respectively. For each case, we perform 1000 Monte Carlo tests to calculate the root mean square error (RMSE), where the pose and feature positions are randomly generated in each test. We let $\sigma=0.5,1$ pixel respectively. The result is plotted in Figure~\ref{fig_consistency}. The RMSEs of noise std, rotation, and translation all decrease linearly with respect to the feature point number in the double-log plot, verifying the consistency of our estimator. The consistent property ensures that our PnP tracking serves as a good initial value for the subsequent iterative refinement as long as the feature is rich.

\begin{figure}[htbp]
    \centering
    \begin{subfigure}[t]{0.24\textwidth}
        \centering
        \includegraphics[width=\textwidth]{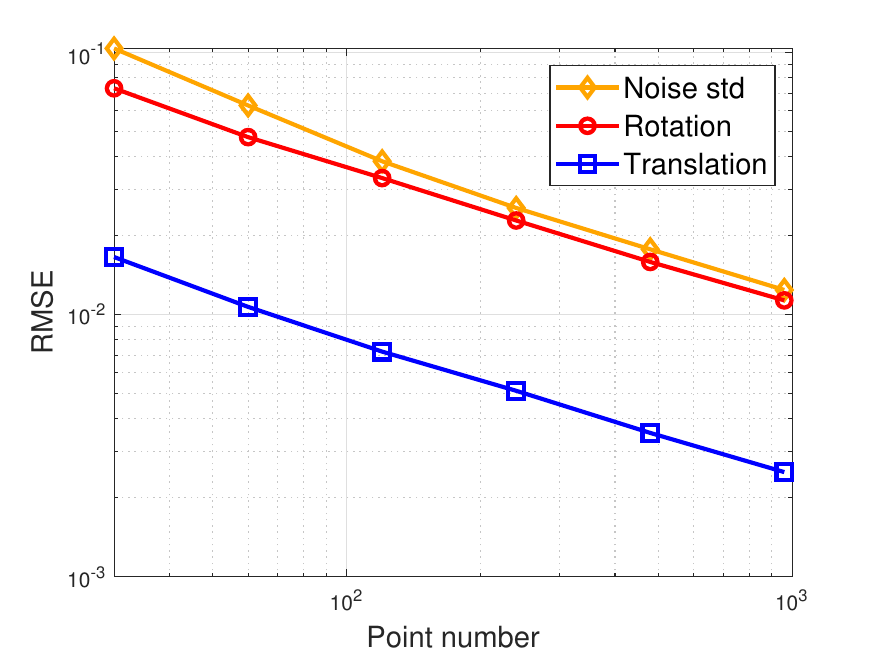} 
        \caption{$\sigma=0.5$ pixel}
        \label{fig_consistency1}
    \end{subfigure}
    \hfill
    \begin{subfigure}[t]{0.24\textwidth}
        \centering
        \includegraphics[width=\textwidth]{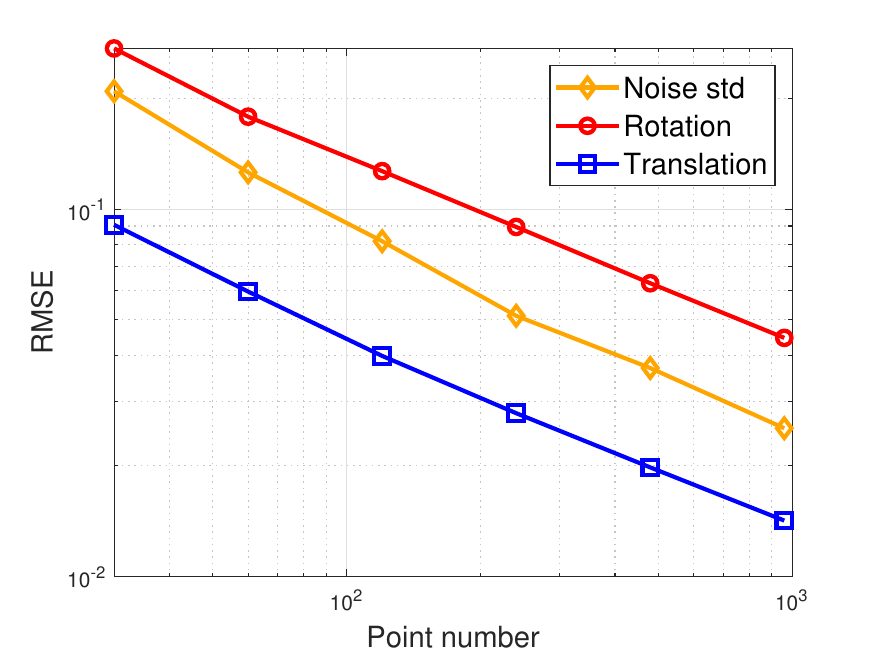} 
        \caption{$\sigma=1$ pixel}
        \label{fig_consistency2}
    \end{subfigure}
    \caption{Consistency of the bias-eliminated PnP estimator. The units of the RMSEs of noise std, rotation, and translation are pixel, $^\circ$, and m.}
    \label{fig_consistency}
\end{figure}

\subsubsection{Rationality of utilizing the latest KF for pose tracking} We deploy two trajectories: a straight line and a circle. Each trajectory contains 500 frames. With no loss of generality, each frame is KF. 3D points are uniformly and randomly distributed in the space such that around 100-200 points are visible in each image. The outlier ratio is set as 2\%. For the compared methods, we suppose they can use the latest $m=2,3$ KFs, respectively. BA is not included in this simulation. We run 10 Monte Carlo tests to calculate the average relative pose error (RPE) and absolute trajectory error (ATE)~\cite{zhang2018tutorial}. The result is shown in Table~\ref{table_different_num_kf}. We also plot the average pose error of each frame for the line trajectory in Figure~\ref{fig_ave_rel_pose_err}.
We see that using the latest KF owns the best accuracy, while integrating more 3D points from the previous three KFs instead exhibits the least precision, especially for the RPE. 
The pose errors of using more KFs also fluctuate more severely. These results demonstrate that the errors of 3D points propagated from previous pose errors will offset the advantages gained from increased feature number, and can even result in deteriorated odometry performance. 

\begin{table}[!htbp]
        % \small
            \centering
       \caption{Average errors of using different numbers of KFs for PnP tracking. The units for errors of $t$ and $R$ are m and $^\circ$.}
            \label{table_different_num_kf}
            \begin{tabular}{cccccc}
           \thickhline
             \multirow{2}{*}{\textbf{Trajectory} } & \multirow{2}{*}{\textbf{KF used}} & \multicolumn{2}{c}{\textbf{ATE}}   & \multicolumn{2}{c}{\textbf{RPE}}\\\cline{3-6}
              & & $t$ & $R$  & $t$ & $R$\\
            \hline
            
            \multirow{4}{*}{Line} & latest & 1.161  &2.452 & 0.046 & 0.048\\ 
            & two & 2.193 &  0.947&  0.176&  0.175\\
             & three &  2.858&  1.859&  0.170&  0.163\\
             & latest+BA &  1.068&  1.307&  0.044&  0.057\\
            \hline

            \multirow{4}{*}{Circle} & latest &   20.313& 6.379&  0.084& 0.084\\ 
            & two &  55.840&  18.784&  0.289&  0.316\\
             & three &  63.209&  20.355&  0.387&  0.395\\
           &latest+BA &  9.415&  1.398&  0.170&  0.080\\
           \thickhline
            \end{tabular}
        \end{table}

        \begin{figure}[!htbp]
    \centering
    \includegraphics[width=0.9\linewidth]{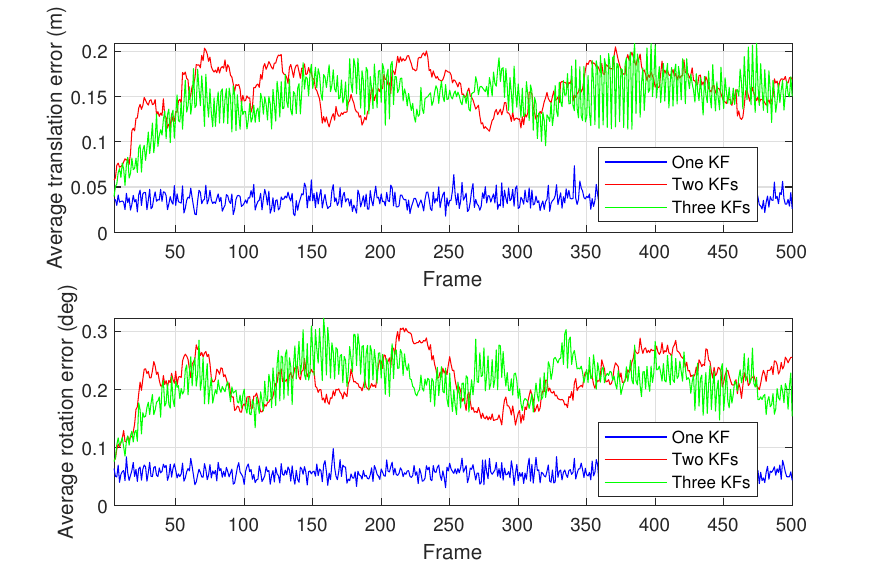}
    \caption{Average pose error of each frame when using different numbers of KFs for the line trajectory.}
    \label{fig_ave_rel_pose_err}
\end{figure}

\begin{table*}[htbp]
\centering
\small
\caption{Comparison of ATE and RPE across different sequences in KITTI dataset. ORB3 denotes ORB-SLAM3 and OV2 represents OV\textsuperscript{2}SLAM. Values highlighted in {\color{blue} \textbf{blue bold}} represent the smallest, and values in {\color{blue!80} blue} denote the second smallest.}
\label{tab:KITTI comparison}
\setlength{\tabcolsep}{6pt}
\begin{tabular}{c|ccc|ccc|>{\hspace{2pt}}c>{\hspace{2pt}}c>{\hspace{2pt}}c|>{\hspace{2pt}}c>{\hspace{2pt}}c>{\hspace{2pt}}c}
\thickhline
\multirow{3}{*}{Sequence} & \multicolumn{6}{c|}{ATE (m)} & \multicolumn{6}{c}{RPE (m)} \\
\cline{2-13}
 & \multicolumn{3}{c|}{Color} & \multicolumn{3}{c|}{Grayscale} & \multicolumn{3}{c|}{Color} & \multicolumn{3}{c}{Grayscale} \\
\cline{2-13}
 & ORB3 & OV2 & Ours & ORB3 & OV2 & Ours & ORB3 & OV2 & Ours & ORB3 & OV2 & Ours \\
\hline
seq00 & {\color{blue}\bfseries 4.042} & 4.676 & {\color{blue!80}4.174} & {\color{blue}\bfseries 4.263} & 4.767 & {\color{blue!80}4.514} & 0.0287 & {\color{blue!80}0.0278} & {\color{blue}\bfseries 0.0262} & 0.0283 & {\color{blue!80}0.0262} & {\color{blue}\bfseries 0.0260} \\
seq02 & {\color{blue!80}9.549} & 11.406 & {\color{blue}\bfseries 5.756} & 7.900 & {\color{blue!80}7.363} & {\color{blue}\bfseries 3.900} & 0.0286 & {\color{blue!80}0.0278} & {\color{blue}\bfseries 0.0257} & 0.0277 & {\color{blue!80}0.0263} & {\color{blue}\bfseries 0.0257} \\
seq03 & {\color{blue!80}3.846} & 4.183 & {\color{blue}\bfseries 0.551} & 1.200 & {\color{blue!80}1.177} & {\color{blue}\bfseries 1.030} & {\color{blue!80}0.0250} & 0.0264 & {\color{blue}\bfseries 0.0148} & 0.0182 & {\color{blue!80}0.0166}  & {\color{blue}\bfseries 0.0158} \\
seq04 & {\color{blue!80}3.160} & 3.453 & {\color{blue}\bfseries 2.328} & {\color{blue}\bfseries 0.213} & 1.306 & {\color{blue!80}1.057} & {\color{blue!80}0.0445} & 0.0487 & {\color{blue}\bfseries 0.0353} & {\color{blue}\bfseries 0.0198} & 0.0239 & {\color{blue!80} 0.0218} \\
seq05 & {\color{blue!80}3.904} & 4.254 & {\color{blue}\bfseries 3.332} & {\color{blue}\bfseries 2.115} & 2.448 & {\color{blue!80}2.403} & {\color{blue!80}0.0264} & 0.0267 & {\color{blue}\bfseries 0.0178} & 0.0166 & {\color{blue!80}0.0163} & {\color{blue}\bfseries 0.0124} \\
seq06 & {\color{blue!80}4.279} & 5.052 & {\color{blue}\bfseries 2.400} & {\color{blue}\bfseries 1.791} & 3.533 & {\color{blue!80}1.859} & {\color{blue!80}0.0360} & 0.0363 & {\color{blue}\bfseries 0.0187} & {\color{blue!80}0.0174} & 0.0183 & {\color{blue}\bfseries 0.0138} \\
seq07 & {\color{blue!80}1.991} & 2.226 & {\color{blue}\bfseries 1.593} & {\color{blue}\bfseries 1.222} & 1.621 & {\color{blue!80}1.281} & 0.0235 & {\color{blue!80}0.0213} & {\color{blue}\bfseries 0.0175} & 0.0166 & {\color{blue!80}0.0124} & {\color{blue}\bfseries 0.0123} \\
seq08 & {\color{blue!80}6.201} & 6.315 & {\color{blue}\bfseries 5.866} & 3.698 & {\color{blue!80}3.590} & {\color{blue}\bfseries 3.430} & 0.0439 & {\color{blue!80}0.0431} & {\color{blue}\bfseries 0.0397} & {\color{blue!80}0.0389} & {\color{blue}\bfseries 0.0380} & 0.0392 \\
seq09 & 6.598 & {\color{blue!80}6.529} & {\color{blue}\bfseries 5.245} & {\color{blue!80}3.193} & 3.760 & {\color{blue}\bfseries 2.169} & {\color{blue!80}0.0324} & 0.0327 & {\color{blue}\bfseries 0.0234} & {\color{blue!80}0.0232} & 0.0249 & {\color{blue}\bfseries 0.0181} \\
seq10 & 4.477 & {\color{blue!80}4.421} & {\color{blue}\bfseries 3.088} & 1.393 & {\color{blue!80}0.655} & {\color{blue}\bfseries 0.638} & 0.0261 & {\color{blue!80}0.0237} & {\color{blue}\bfseries 0.0196} & 0.0211 & {\color{blue!80}0.0181} & {\color{blue}\bfseries 0.0172} \\
\hline
Ave & {\color{blue!80}4.805} & 5.252 & {\color{blue}\bfseries 3.433} & {\color{blue!80}2.699} & 3.022 & {\color{blue}\bfseries 2.228} & 0.0315 & {\color{blue!80}0.0314} & {\color{blue}\bfseries 0.0239} & 0.0228 & {\color{blue!80}0.0221} & {\color{blue}\bfseries 0.0202} \\
\thickhline
\end{tabular}
\end{table*}

\subsubsection{Effect of the epipolar BA}
In addition to PnP pose tracking, this simulation incorporates epipolar BA with a sliding window of four frames. The experimental setup remains the same as in the previous simulation, and two types of trajectories are tested.
The comparison of ATE and RPE is presented in Table~\ref{table_different_num_kf}.
It can be observed that incorporating BA significantly improves ATE, albeit with a slight trade-off in RPE. This is because, without BA, occasional large relative pose errors can propagate through the trajectory, degrading its accuracy and increasing the ATE. By rectifying these occasional large deviations, epipolar BA enhances trajectory consistency and overall accuracy.

\subsection{Experiments on real datasets} \label{experiments_real}
We evaluate the performance of CurrentFeature Odometry on two publicly available datasets: KITTI~\cite{geiger2013vision} and Oxford RobotCar~\cite{maddern20171}. For comparison, we include ORB-SLAM3~\cite{campos2021orb} and OV\textsuperscript{2}SLAM~\cite{ferrera2021ov}, the top two open-source stereo VO approaches on the KITTI dataset. To ensure a fair comparison and focus solely on odometry performance, the loop closure modules in both methods are disabled. For evaluation metrics, we adopt the ATE RMSE and RPE RMSE, which are also used in ORB-SLAM3 and OV\textsuperscript{2}SLAM. 
% \zgy{We do not compare the running time for two main reasons: First, our focus is on the odometry accuracy of the proposed algorithm rather than on optimizing its computational speed; Second, our framework avoids 3D point optimization, which is theoretically more efficient than SLAM pipelines.}

\begin{figure*}[htbp]
    \centering
    \includegraphics[width=\textwidth]{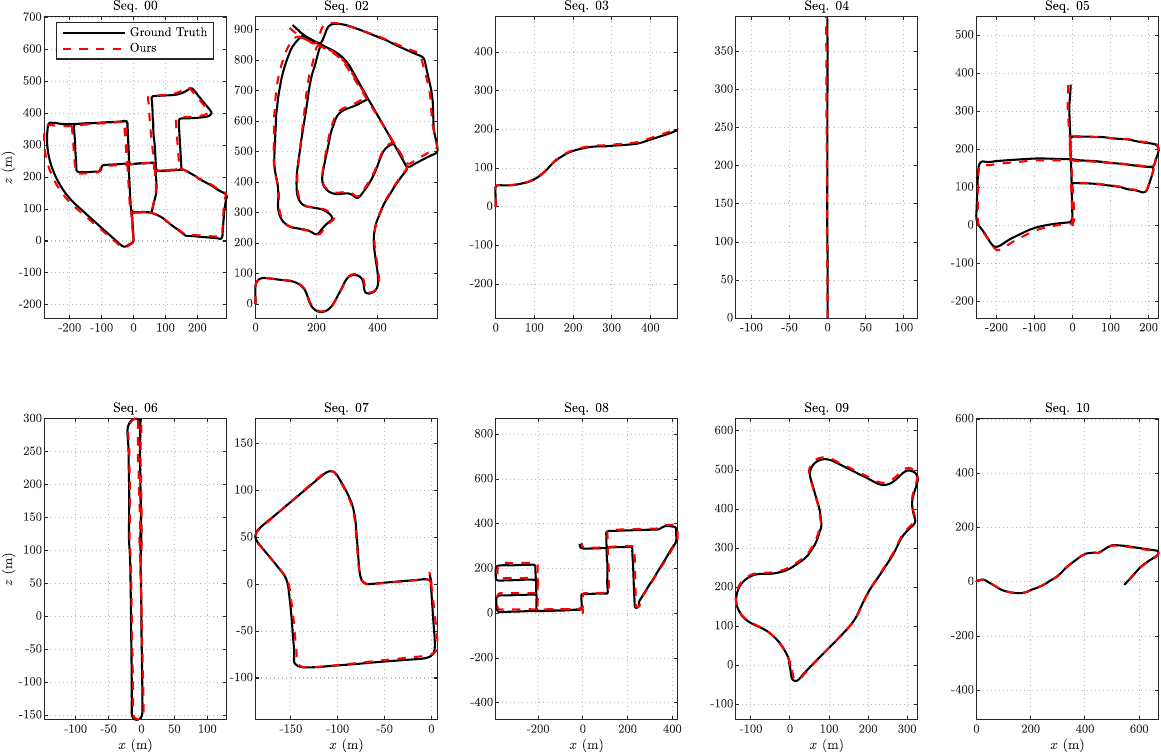}
    \caption{Top-view trajectory of KITTI sequences on grayscale images.}
    \label{fig:kitti_trajectories}
\end{figure*}

\subsubsection{KITTI Dataset} The KITTI dataset is captured in diverse outdoor driving scenarios, spanning urban, suburban, and rural environments. It offers synchronized and rectified stereo image pairs in both color and grayscale formats, recorded at a frame rate of 10 Hz. The dataset also includes ground-truth poses obtained from an OXTS3003 GPS/IMU unit for 11 training sequences (seq00-10), enabling trajectory evaluation and benchmarking. 
It has been testified in~\cite{cvivsic2022soft2} that the stereo rig is not strictly rigid and the relative rotation of the right camera to the left changes as the vehicle moves. Hence, in the epipolar BA, we also optimize the rotation of the rig. 
The parameter setting remains unchanged over all sequences. The TLS thresholds for PnP and epipolar BA are set as $1000$ and $3 \times 10^{-4}$, respectively. The number of OFs used in each epipolar BA is at most $5$.

The ATE and RPE comparison results are summarized in Table~\ref{tab:KITTI comparison}. We see that on average all algorithms perform better on grayscale images compared to color images. This is primarily because color images are more susceptible to severe exposure changes, which can lead to inaccurate feature association or tracking.
As anticipated, our proposed CurrentFeature Odometry consistently outperforms ORB-SLAM3 and OV\textsuperscript{2}SLAM in terms of RPE, with the exception of the grayscale 04 and 08 sequences. Compared to the second-best method, our algorithm achieves a 24\% reduction in average RPE on color sequences. This improvement is attributed to the decoupling of pose and 3D point errors, along with our carefully designed Bias-Eli-W PnP estimator.
Interestingly, our method also achieves better ATE performance compared to SOTA SLAM pipelines. Specifically, it delivers the highest accuracy in 9 out of 10 color sequences, resulting in a 28\% improvement in average ATE compared to the second-best method, ORB-SLAM3. In the grayscale case, CurrentFeature Odometry achieves the best performance in half of the sequences and ranks second in the other half.
Our algorithm's reduced sensitivity to the challenges of color images can be attributed to the enhanced robustness provided by $\ell_1$-norm and TLS techniques. The excellent ATE performance of our method is largely due to the integration of a sliding-window epipolar BA, which improves overall trajectory consistency.
The trajectory visualization is presented in Figure~\ref{fig:kitti_trajectories}. Due to the extremely low drift speed, CurrentFeature Oodmetry's estimated trajectory matches the ground truth well in these large-scale scenes.

\begin{table*}[htbp]
\centering
\small
\caption{Comparison of ATE and RPE across different sequences in Oxford dataset. Values highlighted in {\color{blue} \textbf{blue bold}} represent the smallest, and values in {\color{blue!80} blue} denote the second smallest. The sign - indicates failing to estimate the whole trajectory.}
\label{tab:oxford comparison}
\setlength{\tabcolsep}{4pt}
\begin{tabular}{c|cc c|ccc|ccc}
\thickhline
\multicolumn{1}{c|}{\multirow{2}{*}{Sequence}} & \multicolumn{1}{c}{\multirow{2}{*}{Start frame}} & \multicolumn{1}{c}{\multirow{2}{*}{Stop frame}} & \multicolumn{1}{c|}{\multirow{2}{*}{Length (m)}} & \multicolumn{3}{c|}{ATE (m)} & \multicolumn{3}{c}{RPE (m)} \\
\cline{5-10}
 & & & & ORB3 & OV2 & Ours & ORB3 & OV2 & Ours \\
\hline
2014-05-14-13-59-05 & 1400075963932033 & 1400076150344330 & 945.49 & {\color{blue!80} 34.5311} & 34.6471 & {\color{blue}\bfseries 33.7454} & {\color{blue!80} 0.5126} & 0.5146 & {\color{blue}\bfseries 0.5061 } \\
2014-05-19-12-51-39 & 1400503987511809 & 1400504194609323 & 931.58 & - & {\color{blue!80} 28.9658} & {\color{blue}\bfseries 28.3328} & - & {\color{blue!80} 0.4761} & {\color{blue}\bfseries 0.4682} \\
2014-05-19-13-20-57 & 1400505700597042 & 1400506098919265 & 1989.92 & - & {\color{blue!80} 72.4767}  & {\color{blue}\bfseries 64.9658} & - & {\color{blue!80} 0.5768} & {\color{blue}\bfseries 0.5680} \\
2014-11-14-16-34-33 & 1415985043842007 & 1415985331240621 & 1531.52 & {24.1157} & { \color{blue!80} 15.5798} & { \color{blue}\bfseries 6.2296} & {\color{blue!80}0.5773} & 0.5872 & {\color{blue}\bfseries 0.5674} \\
\hline
\thickhline
\end{tabular}
\end{table*}

% \begin{figure*}[htbp]
%     \centering
%     \begin{subfigure}[t]{0.32\textwidth}
%         \centering
%         \includegraphics[width=\textwidth]{example-image-a} 
%         \caption{Subfigure 1 caption}
%         \label{fig:subfig1}
%     \end{subfigure}
%     \hfill
%     \begin{subfigure}[t]{0.32\textwidth}
%         \centering
%         \includegraphics[width=\textwidth]{example-image-b} 
%         \caption{Subfigure 2 caption}
%         \label{fig:subfig2}
%     \end{subfigure}
%     \hfill
%     \begin{subfigure}[t]{0.32\textwidth}
%         \centering
%         \includegraphics[width=\textwidth]{example-image-c} 
%         \caption{Subfigure 3 caption}
%         \label{fig:subfig3}
%     \end{subfigure}
%     \caption{A figure spanning two columns with three subfigures. Each subfigure has its own caption.}
%     \label{fig:mainfig}
% \end{figure*}

\begin{figure*}[!htbp]
    \centering
    \includegraphics[width=1\linewidth]{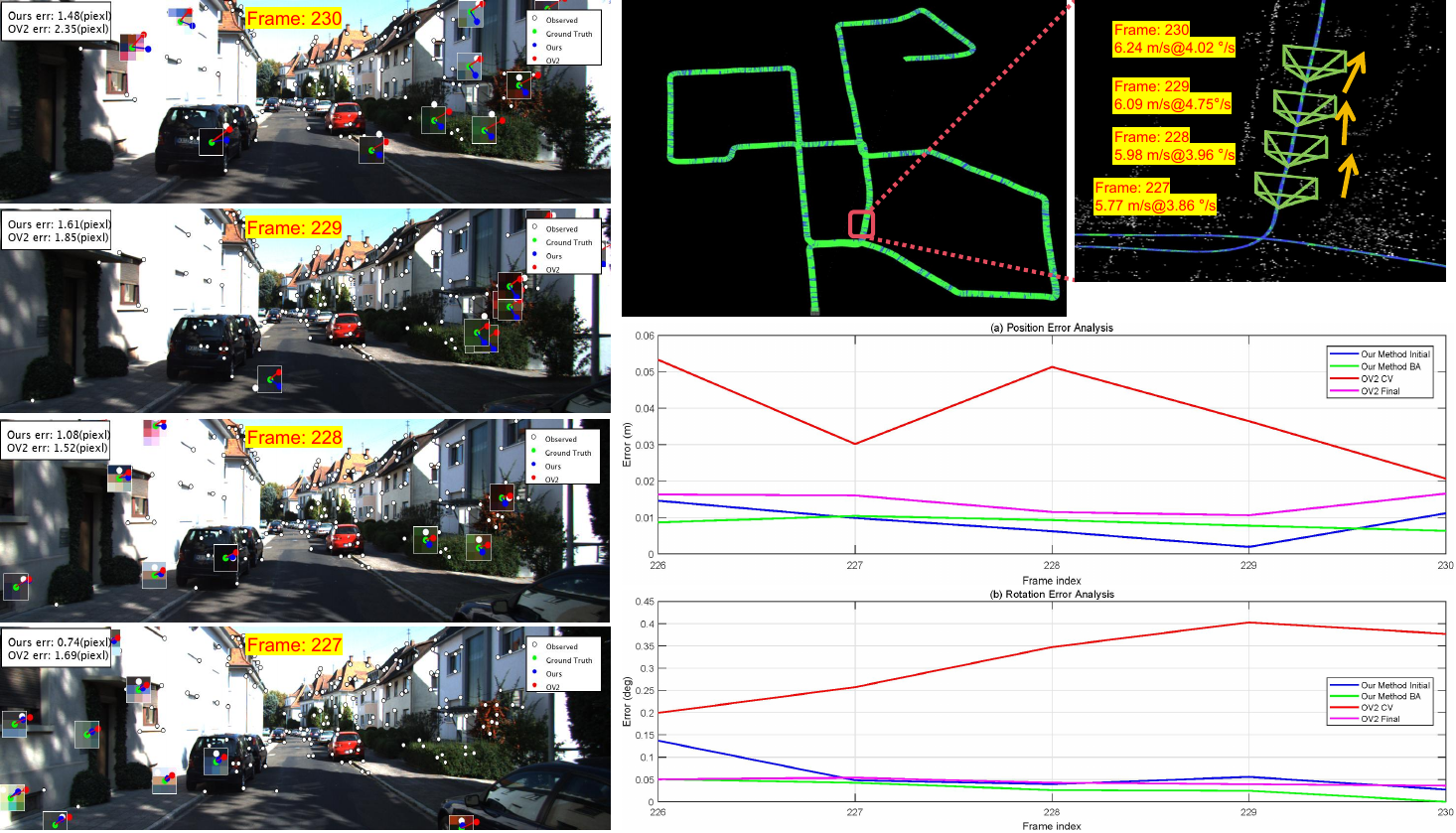}
    \caption{Error analysis on a challenging segment with varying velocity. The figure consists of four key components: the left are sequential images with reprojection results, the middle top shows the bird's-eye view of the trajectory with a zoomed section, the right top presents a detailed visualization of camera motion with velocity annotations, and the bottom plots quantitative errors.}
    \label{error_analysis}
\end{figure*}

\subsubsection{Oxford RobotCar Dataset}
The Oxford dataset was collected in central Oxford using a Point Grey Bumblebee XB3 trinocular stereo camera, capturing images at a resolution of 1280 $\times$ 960 pixels and a frame rate of 16 Hz. The vehicle's ground-truth trajectory was recorded with a GPS/INS system operating at 50 Hz. However, the accuracy of the GPS reception and the fused INS solution varies significantly depending on the environment~\cite{maddern20171}. Similar to SOFT2~\cite{cvivsic2022soft2}, we select four segments of the dataset that provide valid sensor data, reliable GPS/INS trajectories, and high-quality images for evaluation. Additionally, to eliminate irrelevant visual information such as the sky and the vehicle chassis, we cropped 140 pixels from both the top and bottom of the images.
The parameter settings remain the same as those in the KITTI trial across all four sequences.
The ATE and RPE comparisons are presented in Table~\ref{tab:oxford comparison}. Our algorithm achieves the best performance across all sequences for both metrics. It is noteworthy that compared to the KITTI dataset, the Oxford dataset poses greater challenges due to issues such as strong exposure. Consequently, ORB-SLAM3 fails to estimate the complete trajectory in two sequences. In contrast, our $\ell_1$-TLS-enhanced method demonstrates superior robustness, achieving significantly smaller ATE in the fourth sequence.

\subsection{Robustness against erratic motion}
Unlike OV\textsuperscript{2}SLAM which relies on a constant velocity (CV) motion model for pose prediction, our method directly estimates camera pose through PnP using 3D-2D correspondences, without making any assumptions about the motion model. The CV assumption in OV\textsuperscript{2}SLAM may introduce additional errors when the actual motion deviates from uniform motion.

% \textbf{is CV really enough?}
To validate this, we analyze a challenging segment (frames 226-230) from KITTI color sequence 00. As shown in Figure~\ref{error_analysis}, within a brief 0.5-second span (5 frames), this segment features increasing linear velocity from 5.7 m/s to 6.2 m/s and fluctuating angular velocity up to 5 deg/s. Under such dynamic conditions, OV\textsuperscript{2}SLAM's CV prediction shows some instability, with position errors reaching 0.05 m and rotation errors accumulating up to 0.4 degrees. In comparison, our method, by directly solving PnP without motion assumptions, maintains relatively stable errors throughout this challenging segment. 
The reprojection results, which project the triangulated 3D points onto the CF using the ground-truth pose, our PnP pose, and the CV prediction pose, also demonstrate the superiority of our PnP estimator.
After BA optimization, our method achieves position errors below 0.01 m and rotation errors under 0.05 degrees, showing improved accuracy over OV\textsuperscript{2}SLAM's initial and final results.

\section{Conclusion} \label{conclusion}
In this paper, we revisited stereo VO and proposed a consistent PnP-enabled framework, CurrentFeature Odometry. Our method leverages triangulated points from the current keyframe for PnP tracking, effectively breaking the coupling between the pose and 3D point errors. Thanks to the decoupling, we accurately modeled the uncertainties of 3D points based on statistical theory and error propagation formulas. These uncertainties are then incorporated for both bias elimination and weighting during optimization.
We proved that our bias-eliminated PnP pose estimate converges to the true value with a rate of $1/\sqrt{n}$, which serves as a good initial value for subsequent iterative refinement. Additionally, the inclusion of an epipolar BA further enhances trajectory consistency.
Experimental results showed that CurrentFeature Odometry outperforms SOTA algorithms in terms of ATE and RPE in large-scale environments, while also exhibiting greater robustness to unpredictable motions.

% \section*{Acknowledgment} 

%% Use plainnat to work nicely with natbib. 

%\bibliographystyle{plainnat}
\bibliographystyle{unsrtnat}
\bibliography{references}

\end{document}